\documentclass{article}[a4paper]
\usepackage[margin=1in]{geometry}
\usepackage{graphicx}%
\usepackage{multirow}%
\usepackage{amsmath,amssymb,amsfonts}%
\usepackage{amsthm}%
\usepackage{mathrsfs}%
\usepackage{bbm}
\usepackage[title]{appendix}%
\usepackage{xcolor}%
\usepackage{textcomp}%
\usepackage{manyfoot}%
\usepackage{booktabs}%
\usepackage{algorithm}%
\usepackage{algorithmicx}%
\usepackage{algpseudocode}%
\usepackage{listings}%
\usepackage{hyperref}

\usepackage{subcaption}
\usepackage{bm}

\newtheorem{theorem}{Theorem}
\newtheorem{lemma}{Lemma}
\theoremstyle{definition}
\newtheorem{remark}{Remark}%

\usepackage[backend=biber]{biblatex}
\providecommand{\keywords}[1]
{
  \textbf{\textit{Key words: }} #1
}

\providecommand{\subjclass}[1]
{
  \textbf{\textit{MSC2020: }} #1
}

\begin{document}


\title{A hierarchical Vovk-Azoury-Warmuth forecaster with discounting for online regression in RKHS}

\author{
  Dmitry B. Rokhlin\thanks{The research was supported by the Regional Mathematical Center of the Southern Federal University with the Agreement no. 075-02-2025-1720 of the Ministry of Science and Higher Education of Russia.} \\
  Institute of Mathematics, Mechanics and Computer Sciences of the\\
  Southern Federal University \\
  Regional Scientific and Educational Mathematical Center of the\\
  Southern Federal University \\
  \texttt{dbrohlin@sfedu.ru}
}

 \date{} 

\maketitle

\begin{abstract}
We study the problem of online regression with the unconstrained quadratic loss against a time-varying sequence of functions from a Reproducing Kernel Hilbert Space (RKHS). Recently, Jacobsen and Cutkosky (2024) introduced a discounted Vovk-Azoury-Warmuth (DVAW) forecaster that achieves optimal dynamic regret in the finite-dimensional case. In this work, we lift their approach to the non-parametric domain by synthesizing the DVAW framework with a random feature approximation. We propose a fully adaptive, hierarchical algorithm, which we call H-VAW-D (Hierarchical Vovk-Azoury-Warmuth with Discounting),
that learns both the discount factor and the number of random features. We prove that this algorithm, which has a per-iteration computational complexity of $O(T\ln T)$, achieves an expected dynamic regret of $O(T^{2/3}P_T^{1/3} + \sqrt{T}\ln T)$, where $P_T$ is the functional path length of a comparator sequence.
\end{abstract}



\section{Introduction}\label{sec:introduction}
We consider the problem of online linear regression with the quadratic loss in an unconstrained, adversarial setting. An online algorithm sequentially observes feature vectors $x_t \in \mathbb{R}^d$ and must produce a prediction $\hat{y}_t \in \mathbb{R}$ before the true outcome $y_t \in \mathbb{R}$ is revealed. The goal is to minimize regret against a class of comparator functions. In the foundational \emph{static regret} setting, where the algorithm competes against the best single linear predictor in hindsight, the Vovk-Azoury-Warmuth (VAW) forecaster achieves a logarithmic regret of $O(d \ln T)$ \cite{Vovk2001, Azoury2001}, and is known to be minimax optimal:~\cite{Gaillard2019uniform}.

While this metric is central to the field, many real-world applications are inherently non-stationary. This challenge is captured by the more demanding notion of \emph{dynamic regret}, introduced in \cite{Zinkevich2003online}. This metric compares the learner's performance against a sequence of comparators $\mathbf{u} = (u_1, \dots, u_T)$, one for each time step, thereby tracking the performance of the best possible dynamic strategy. Achieving sublinear dynamic regret is only feasible when the comparator sequence exhibits some form of regularity, commonly measured by its total path length, $P_T(\mathbf u) = \sum_{t=2}^T \|u_t - u_{t-1}\|_2$.
For general convex losses, under boundedness assumptions, the work \cite{Zhang2018adaptive} established that the minimax regret rate is $\Theta(\sqrt{T(1+P_T)})$, providing both a matching upper bound with their Ader algorithm and a corresponding lower bound. For the unconstrained quadratic loss recently Jacobsen and Cutkosky~\cite{Jacobsen2024} introduced a discounted VAW (DVAW) forecaster.  By incorporating a forgetting factor, their algorithm achieves dynamic regret of the form $O(d\ln T \lor \sqrt{d P_T T})$, representing the state-of-the-art for this problem.

A natural and important extension is to move to non-parametric domain and study a competition against a rich class of non-linear functions taken from Reproducing Kernel Hilbert Spaces (RKHS). Vovk~\cite{Vovk2006} established a universal static regret bound of $O(\|f\|_{\mathcal H}\sqrt{T})$ for a comparator $f\in\mathcal H$. By leveraging the geometry of the observed data this bound can be improved to $O(d_{\operatorname{eff}} \ln T)$ \cite{Jezequel2019efficient}, where $d_{\operatorname{eff}}$ is a data-dependent effective dimension that measures the spectral complexity of the data's kernel matrix. This bound is stronger whenever the data lies in a low-dimensional subspace of the RKHS. For general classes of functions a refined, non-algorithmic analysis of \cite{Rakhlin2014online}, based on the concept of ``sequential entropy'', established the optimal minimax rates for this problem, revealing a phase transition in regret behavior dependent on the complexity of the function class.

However, achieving these theoretical regret bounds requires algorithms that are computationally prohibitive. The exact kernelized VAW, for instance, requires $O(t^2)$ operations at each step $t$, leading to a total computational complexity of $O(T^3)$. This ``curse of kernelization'' \cite{Hoi2021online} has motivated a strong focus on scalable online kernel learning, with solutions that can be broadly categorized into two families \cite{Hoi2021online}: budget maintenance strategies, controlling the number of support vectors, and functional approximation, where the core idea is to approximate the kernel function using a low-dimensional feature map. The work~\cite{Jezequel2019efficient} provides a modern analysis within the functional approximation paradigm, specifically for the kernelized VAW algorithm. They propose and analyze efficient algorithms based on Nystr{\"o}m and Taylor series approximations. An alternative and highly successful functional approximation technique is the use of  random Fourier features (RFF)~\cite{Rahimi2007}. In multikernel setting this method was applied in \cite{Sahoo2019large, Shen2019random} under boundedness assumptions and in \cite{Rokhlin2025random} for the unconstrained quadratic loss.

These results address the \emph{static} regret setting. The challenge of achieving optimal \emph{dynamic} regret against comparators in RKHS spaces with a computationally efficient method remained an open problem. Our work is inspired by 
\cite{Jacobsen2024} and relies heavily on its basic results. We address the mentioned challenge by synthesizing the dynamically adaptive discounting framework of \cite{Jacobsen2024} with the computational efficiency of the random feature approximation. The goal is to develop an algorithm that competes against a sequence of functions $\mathbf{f} = (f_1, \dots, f_T)$ from an RKHS, achieving optimal dynamic regret that adapts to the functional path length $P_T(\mathbf{f}) = \sum_{t=1}^{T-1} \|f_{t+1} - f_t\|_{\mathcal{H}}$. 

The proposed hierarchical algorithm, hereafter referred to as H-VAW-D (Hierarchical Vovk-Azoury-Warmuth with Discounting), achieves an expected dynamic regret of 
\begin{equation} \label{1.1}
  O\left(T^{2/3}P_T^{1/3}(\mathbf f)+\sqrt T\ln T\right)
\end{equation}
for non-parametric dynamic regression in RKHS spaces. Our approach is fully adaptive, learning not only the optimal discount factor $\gamma$ but also the number of random features $m$, all without prior knowledge of $P_T(\mathbf{f})$. The architecture of the algorithm is hierarchical and consists of three nested levels of online learning:
\begin{itemize}
    \item \emph{Level 1 (base experts)}. At the lowest level, for a fixed number of random features $m$, we run a collection of DVAW forecasters. Each forecaster in this collection operates on the same $m$-dimensional random feature space but uses a different, fixed discount factor $\gamma$. 
    This creates a set of ``base experts,'' each specialized in tracking a comparator sequence with a specific degree of non-stationarity.
    \item \emph{Level 2 (feature experts)}. For each choice of feature dimension $m$ from a predefined grid, we employ a meta-algorithm that aggregates the predictions from the corresponding collection of DVAW base experts. This meta-algorithm, itself a VAW forecaster, learns the optimal discount factor $\gamma$ on-the-fly for that specific feature dimension $m$. 
    The output of this level is a set of ``feature experts'', representing the best possible dynamic predictor for its given feature space.
    \item \emph{Level 3 (top-level aggregation)}. Finally, a top-level VAW meta-algorithm aggregates the predictions of all the feature experts from Level 2. This final layer adaptively learns the optimal number of random features $m$ required to approximate the true functional comparator.
\end{itemize}
This multi-level, fully adaptive structure allows our algorithm to dynamically tune all its crucial hyperparameters --- the discount factor $\gamma$ and the feature dimension $m$ --- without any prior knowledge of the problem's non-stationarity or the complexity of the comparator sequence. The computational complexity of the algorithm is $O(T\ln T)$ per iteration.

The paper is organized as follows. In Section \ref{sec:2}, we provide the necessary background on reproducing kernel Hilbert spaces, the random feature approximation, and the discounted Vovk-Azoury-Warmuth (DVAW) algorithm which is central to our method. 
In Section~\ref{sec:3}, we present our main results. We begin with Lemma~\ref{lem:3}, a technical result that bounds the variation of the loss in the random feature space by the path length of the functional comparators in the original RKHS. Then we derive an oracle inequality for expected regret with simplified, in comparison to \cite{Jacobsen2024}, dependence on the discounting factor $\gamma$. In Lemma~\ref{lem:4} we use this inequality to derive a regret bound for a meta-algorithm, which adaptively learns the optimal discount factor for a fixed number of random features. Finally, in Theorem~\ref{th:2}, we present our main hierarchical algorithm, H-VAW-D, and its regret analysis. We conclude in Section~\ref{sec:conclusion} with a summary of our contributions and directions for future research.

\section{Preliminaries} \label{sec:2}
Let $\mathcal X$ be any set. Recall that a reproducing kernel Hilbert space (RKHS) is a Hilbert space $\mathcal{H}$ of functions $f: \mathcal{X} \rightarrow \mathbb{R}$ such that any evaluation functional  $\delta_x: f \mapsto f(x)$, $x\in\mathcal X$ is continuous. We focus on the class of kernels that admit a random feature representation, as this structure forms the theoretical basis for our work.
Specifically, let $\Theta$ be a measurable space, and let $\mathsf P$ be a probability measure on its $\sigma$-algebra. Consider a feature map $\phi:\mathcal X\times\Theta\to [-a,a]$ that is measurable with respect to its second argument, and put
\[k(x,y) = \int \phi(x;\theta)\phi(y;\theta)\mathsf P(d\theta)=\langle\phi(x;\cdot),\phi(y;\cdot)\rangle_{L^2(\mathsf P)}.\]
The function $k$ is a kernel, that is, it is symmetric and positive semidefinite:
\[ k(x,y)=k(y,x),\quad \sum_{i,j=1}^n\alpha_i\alpha_jk(x_i,x_j)\ge 0,\quad x_i\in\mathcal X, \alpha_i\in\mathbb R. \]
By the Moore-Aronszajn theorem \cite{Aronszajn1950theory} there exists a unique RKHS with the reproducing kernel $k$:
\[ \mathcal H=\{f:f(x)=\langle f,k(\cdot,x)\rangle_\mathcal H\}. \]
Note that elements of $\mathcal H$ are uniformly bounded. By the Cauchy-Schwarz inequality and the reproducing property:
\begin{equation} \label{2.0}
|f(x)|\le \|f\|_\mathcal H \|k(\cdot,x)\|_\mathcal H=\|f\|_{\mathcal H} \sqrt{\langle k(\cdot,x),k(\cdot,x)\rangle}= \|f\|_{\mathcal H} \sqrt{k(x,x)}\le a\|f\|_\mathcal H.    
\end{equation}

\begin{lemma} \label{lem:1}
For any $f \in \mathcal{H}$, there exists a function $\alpha_f \in L^2(\mathsf{P})$ such that:
\begin{itemize}
 \item  $f(x) = \int_\Theta \alpha_f(\theta)\phi(x;\theta)\mathsf{P}(d\theta)=\mathsf E(\alpha_f\phi(x;\cdot))$ for all $x \in \mathcal{X}$.
\item $\|f\|_{\mathcal{H}}^2 = \int_\Theta \alpha_f^2(\theta) \mathsf{P}(d\theta) =\|\alpha_f\|_{L^2(\mathsf{P})}^2$.
\end{itemize}
Furthermore, this $\alpha_f$ is unique if we require it to be in $S = \overline{\mathrm{span}\{\phi(x;\cdot) : x\in \mathcal{X}\}}^{L^2(\mathsf{P})}$.    
\end{lemma}
\begin{proof}
Let $\mathcal{F}$ be the space of functions from $\mathcal X$ to $\mathbb{R}$. Define a linear map $\mathcal T: L^2(\mathsf P) \to \mathcal{F}$ by
$$ (\mathcal T \alpha)(x) = f_\alpha(x) := \int_{\Theta} \alpha(\theta) \phi(x;\theta) \mathsf P(d\theta). $$
Clearly, $S$ is a closed subspace of $L^2(\mathsf P)$.
The null space of $\mathcal T$ coincides with the orthogonal complement $S^\perp$ of $S$, since the condition $\mathcal T\alpha = 0$ means that $(\mathcal T\alpha)(x) = \langle \alpha, \phi(x;\cdot) \rangle_{L^2(\mathsf{P})} = 0$ for all $x \in \mathcal{X}$. 

Put $\mathcal H'=\mathcal T(S)$. Since $\mathcal T:S \to \mathcal{H}'$ is an injective linear map, we can define an inner product on $\mathcal{H}'$ by  ``pulling back" the inner product from $S$. For $f_1 = \mathcal T\alpha_1$ and $f_2 = \mathcal T\alpha_2$ (with $\alpha_1, \alpha_2 \in S$), define $\langle f_1, f_2 \rangle_{\mathcal{H}'} := \langle \alpha_1, \alpha_2 \rangle_{L^2(\mathsf{P})}$. This makes $\mathcal T$ an isometric isomorphism from $S$ to $\mathcal{H}'$. Since $S$ is a Hilbert space, $\mathcal{H}'$ is also a Hilbert space (i.e., it is complete).

To show that $k(\cdot,y)\in\mathcal H'$ take $\alpha_y=\phi(y;\cdot)\in S$ and note that
\[ (\mathcal T\alpha_y)(x)=\langle \alpha_y, \phi(x;\cdot) \rangle_{L^2(\mathsf{P})} = \int_\Theta \phi(y;\theta)\phi(x;\theta)\mathsf{P}(d\theta) = k(x,y),\]
which means that $k(\cdot,y) = \mathcal T\alpha_y\in\mathcal H'$. Let us check the reproducing property. For any $f=\mathcal T\alpha$, $\alpha\in S$ we have:
\[ \langle f, k(\cdot,y) \rangle_{\mathcal{H}'} = \langle T\alpha, T\alpha_y \rangle_{\mathcal{H}'}= \langle \alpha, \alpha_y \rangle_{L^2(\mathsf{P})}= \int_\Theta \alpha(\theta)\phi(y;\theta)\mathsf{P}(d\theta)= (T\alpha)(y) = f(y).\]

We have constructed a Hilbert space $\mathcal{H}'$ of functions on $\mathcal{X}$ which has $k(x,y)$ as its reproducing kernel. By the Moore-Aronszajn theorem, the RKHS with a given positive definite kernel is unique: $\mathcal H'=\mathcal H$.

For any $f \in \mathcal{H}$, there is a unique $\alpha_f \in S$ such that $f = \mathcal T\alpha_f$, and both equalities in the assertion of lemma hold true by definitions. If we  allow $\alpha_f$ to be any element in $L^2(\mathsf{P})$ such that $f=\mathcal T\alpha_f$, it would not be unique: any element in $S^\perp$ can be added to $\alpha_f$. The function $\alpha_f \in S$ is the one with the minimum $L^2(\mathsf{P})$-norm among all possible representations.
\end{proof}

Draw $m$ i.i.d. samples $\theta_1,\dots,\theta_m$ from the distribution $\mathsf P$ and define the random feature map
\[ \Phi_m(x)=\frac{1}{\sqrt m}(\phi(x;\theta_1),\dots,\phi(x;\theta_m))^\top\in\mathbb R^m.\]
The inner product of these random feature maps is an unbiased estimate of the original kernel $k$:
\[ \mathsf E\langle \Phi_m(x),\Phi_m(y)\rangle=\frac{1}{m}\mathsf E\sum_{i=1}^m\phi(x;\theta_i)\phi(y;\theta_i)=k(x,y).\]
However, we are mostly interested in unbiased estimates for elements of $\mathcal H$.

\begin{lemma} \label{lem:2}
Let $f\in\mathcal H$, and let $\alpha_f\in S$ be its representation from Lemma~\ref{lem:1}, such that $f(x)=\mathsf E[\alpha_f\phi(x;\cdot)]$. Put
 \[ \widehat w=\frac{1}{\sqrt m}\left(\alpha_f(\theta_i),\dots,\alpha_f(\theta_m) \right)^\top,\]
 where $\theta_i\sim \mathsf P$ are i.i.d. random variables. Then
  \begin{align} \label{2.1}
 \mathsf E(\langle\widehat w,\Phi_m(x)\rangle-f(x))^2\le a^2 \frac{\|f\|^2_\mathcal H}{m} ,\quad \mathsf E\|\widehat w\|_2^2=\|f\|^2_\mathcal H.
 \end{align}
Moreover, for any $y\in\mathbb R$,
 \begin{align}
  \mathsf E \left( \langle \widehat w,\Phi_m(x)\rangle-y\right)^2 \le \frac{a^2}{m}\|f\|_{\mathcal H}^2+\left(f(x)-y\right)^2. \label{2.2}
\end{align}
\end{lemma}
\begin{proof}
The estimate $\langle \widehat w,\Phi_m(x)\rangle$ of $f(x)$ is unbiased:
\begin{align*}
 \mathsf E\langle \widehat w,\Phi_m(x)\rangle=\frac{1}{m}\sum_{i=1}^m \mathsf E\left(\alpha_f(\theta_i)\phi(x,\theta_i) \right)=f(x).    
\end{align*}
Compute the variance of this estimate:
\begin{align*}
&\mathsf E\left(\frac{1}{m}\sum_{i=1}^m \alpha_f(\theta_i)\phi(x;\theta_i)-f(x) \right)^2=\frac{1}{m}\mathsf E\left(\alpha_f(\theta_1)\phi(x;\theta_1)-f(x)\right)^2\nonumber\\
&\le\frac{1}{m}\mathsf E\left(\alpha_f(\theta_1)\phi(x;\theta_1)\right)^2=\frac{1}{m}\int\alpha_f^2(\theta)\phi^2(x;\theta)\mathsf P(d\theta)\le \frac{a^2}{m}\int\alpha_f^2(\theta)\mathsf P(d\theta) = a^2\frac{\|f\|_{\mathcal H}^2}{m}. \end{align*}
The proof of the equality in \eqref{2.1} is also elementary:
\[ \mathsf E\|\widehat w\|^2_2=\frac{1}{m}\sum_{i=1}^m\mathsf E\alpha_f^2(\theta_i)=\|f\|_{\mathcal H}^2. \]    
To prove \eqref{2.2} note that
\begin{align*}
  \mathsf E \left( \langle \widehat w,\Phi_m(x)\rangle-y\right)^2  - \left( f(x)-y\right)^2&=\mathsf E \left( \langle \widehat w,\Phi_m(x)\rangle\right)^2-f^2(x)\nonumber\\
  &=\mathsf E \left( \langle \widehat w,\Phi_m(x)\rangle-f(x)\right)^2\le \frac{a^2}{m}\|f\|_{\mathcal H}^2, \label{3.2}
\end{align*}
since the estimate $\langle \widehat w,\Phi_m(x)\rangle$ of $f(x)$ is unbiased.
\end{proof}

For any data points $(x_t,y_t)\in\mathbb R^d\times\mathbb R$, $|y_t|\le Y$ consider the sequence of quadratic loss functions
\[ \ell_t(w)=\frac{1}{2}\left(\langle w,\Phi_m(x_t)\rangle-y_t \right)^2,\]
corresponding to the linear regression on random features $\Phi_m$. For a sequence $w_t\in\mathbb R^m$ of decision vectors define the dynamic regret w.r.t. to a  comparator sequence $u_t\in\mathbb R^m$:
\[
R_T(\mathbf{u}) = \sum_{t=1}^T (\ell_t(w_t) - \ell_t(u_t)),\quad \bm u=(u_1,\dots,u_T).
\]
When $u_t=u$ this definition reduces to the definition of the static regret, which we will denote by $R_T(u)$.

To construct $w_t$ we will apply the discounted Vovk-Azoury-Warmuth (DVAW) algorithm, recently introduced in \cite{Jacobsen2024}, over the random features $\Phi_m$. It is assumed that before making prediction $\widehat y_t$ of $y_t$, the forecaster gets a feature vector $x_t$ and a hint $\widetilde y_t$, $|\widetilde y_t|\le \widetilde Y$. Let $\gamma\in (0,1]$, $\lambda >0$, $\widetilde y_1 = 0$. Define 
\[ h_t(w)=\frac{1}{2}(\langle w,\Phi_m(x_t)\rangle-\widetilde y_t)^2,\quad \ell_0(w)=\frac{\lambda}{2}\|w\|_2^2. \]
The regression coefficients are updated by
\[ w_t=\arg \min_{w \in \mathbb{R}^m}\left\{h_t(w) + \gamma \sum_{s=0}^{t-1} \gamma^{t-1-s} \ell_s(w)\right\}, \]
Here we present the algorithm in an FTRL (Follow the Regularized Leader) form. It can also be written in an OMD (Online Mirror Descent) form, which is important for the proofs in \cite{Jacobsen2024}. Put
\[\Sigma_0=\lambda I,\quad \Sigma_t=\Phi_m(x_t)\Phi_m(x_t)^\top+\gamma\Sigma_{t-1}.\]
Then $w_t$ can be written explicitly as  
\[ w_t=\Sigma_t^{-1} \left[ \tilde{y}_t \Phi_m(x_t) + \gamma \sum_{s=1}^{t-1} \gamma^{t-1-s} y_s \Phi_m(x_s)\right],
\]
see \cite[Proposition A.1]{Jacobsen2024}. Note that by applying the Woodbury matrix identity \cite{Hager1989updating}:
$$ (A + u v^\top)^{-1} = A^{-1} - \frac{A^{-1}u v^\top A^{-1}}{1 + v^\top A^{-1} u} $$
to $A = \gamma \Sigma_{t-1}$, $u = v= \Phi_m(x_t)$,
we get an easily implementable recursion for $\Sigma_t^{-1}$:
\begin{align}
\Sigma_t^{-1} &= \left(\gamma \Sigma_{t-1}+\Phi_m(x_t) \Phi_m(x_t)^\top\right)^{-1}=
\left( \frac{1}{\gamma} \Sigma_{t-1}^{-1} \right) - \frac{\left( \frac{1}{\gamma} \Sigma_{t-1}^{-1} \right) \Phi_m(x_t) \Phi_m(x_t)^\top \left( \frac{1}{\gamma} \Sigma_{t-1}^{-1} \right)}{1 + \Phi_m(x_t)^\top \left( \frac{1}{\gamma} \Sigma_{t-1}^{-1} \right) \Phi_m(x_t)}\nonumber\\
&=\frac{1}{\gamma} \left( \Sigma_{t-1}^{-1} - \frac{\Sigma_{t-1}^{-1} \Phi_m(x_t) \Phi_m(x_t)^\top \Sigma_{t-1}^{-1}}{\gamma + \Phi_m(x_t)^\top \Sigma_{t-1}^{-1} \Phi_m(x_t)} \right).  \label{inverse_Sigma_update}  
\end{align}

For $\widetilde y_t=0$ and $\gamma=1$ the discounted VAW (DVAW) algorithm coincides with the standard VAW. For $\gamma=0$ DVAW predicts $y_t$ by convention (even if $\Phi_m(x_t)=0$). We will essentially use the following result.

\begin{theorem}[\cite{Jacobsen2024}, Theorem 3.1] \label{th:1}
For any sequence of comparators $\mathbf{u} = (u_1, \dots, u_T)$, $u_t\in\mathbb{R}^m$, the discounted VAW forecaster guarantees
\begin{align*}
R_T(\mathbf{u}) &\le \frac{\gamma\lambda}{2} \|u_1\|_2^2 + \frac{m}{2} \max_{1\le t\le T} \Delta_t^2 \ln \left( 1 + \frac{\sum_{t=1}^T \gamma^{T-t} \|\Phi_m(x_t)\|_2^2}{\lambda m} \right) \\
& + \gamma \sum_{t=1}^{T-1} [F_t^\gamma(u_{t+1}) - F_t^\gamma(u_t)] + \frac{m}{2} \Delta_{1:T}^2\ln(1/\gamma) 
\end{align*}
where $F_t^\gamma(w) = \gamma^t \frac{\lambda}{2} \|w\|_2^2 + \sum_{s=1}^t \gamma^{t-s} \ell_s(w)$,
\[ \Delta_t^2=(y_t - \tilde{y}_t)^2,\quad \Delta_{1:T}^2=\sum_{t=1}^T (y_t - \tilde{y}_t)^2.\]
\end{theorem}

The static regret estimate for the VAW algorithm is recovered from Theorem \ref{th:1} by letting $\gamma=1$ and $u_t=u$:
\begin{align} \label{2.4}
    R_T(u)\le \frac{\lambda}{2}\|u\|^2+\frac{m}{2}\max_{1\le t\le T}\Delta_{t}^2\ln\left( 1 + \frac{\sum_{t=1}^T \|\Phi_m(x_t)\|_2^2}{\lambda m} \right).
\end{align}
For $\widetilde y_t=0$ it reduces to the standard bound with $\max_{1\le t\le T}\Delta_{t}^2\le Y^2$ given in \cite[Theorem 11.8]{Cesa2006prediction}, \cite[Theorem 1]{Jacobsen2024}.

\section{Main result}\label{sec:3}
Our goal is to bound the dynamic regret with respect to a sequence of functions $\mathbf f=(f_1, \dots, f_T)$ from an RKHS $\mathcal{H}$. The regret is defined as the cumulative difference between the loss of our algorithm's predictions and the loss of the functional comparator sequence:
\[ R_T(\mathbf f)=\sum_{t=1}^T\ell_t(w_t)-\frac{1}{2}\sum_{t=1}^T(f_t(x_t)-y_t)^2.\]
To bridge the gap between the DVAW algorithm, which operates in $\mathbb{R}^m$, and the functional comparator in $\mathcal{H}$, we define a sequence of $m$-dimensional random vector comparators $\widehat{\mathbf{w}} = (\widehat{w}_1, \dots, \widehat{w}_T)$. For each $f_t \in \mathcal{H}$ with representation $\alpha_t \in S$ (from Lemma~\ref{lem:1}), we define its random feature counterpart $\widehat{w}_t \in \mathbb{R}^m$ as:
\[ \widehat w_t=\frac{1}{\sqrt m}\left(\alpha_t(\theta_1),\dots,\alpha_t(\theta_m) \right)^\top,\]
where $\theta_i\sim \mathsf P$ are i.i.d. random variables used in the definition of $\Phi_m(x)$. We assume that the comparator functions are uniformly bounded in norm, i.e., $\|f_t\|_\mathcal H\le R$ for all $t$.

The following lemma is crucial, as it bounds the change in the instantaneous loss of the finite-dimensional vector comparators by the change in the functional comparators.
\begin{lemma} \label{lem:3} Let $|y_t|\le Y$, $\|f_t\|_{\mathcal H}\le R$. Then for any $s$,
\begin{align*}
  \mathsf E (\ell_s(\widehat w_{t+1})-\ell_s(\widehat w_t))\le \rho_m\|f_{t+1}-f_t\|_{\mathcal H},\quad \rho_m=a(aR+Y)+\frac{2Ra^2}{m}.
\end{align*}   
\end{lemma}
\begin{proof}
As in the proof of Lemma \ref{lem:1},
\begin{align*}
\mathsf E(\langle \widehat w,\Phi_m(x)\rangle-f_t(x))^2&=\mathsf E\left(\frac{1}{m}\sum_{i=1}^m \alpha_t(\theta_i)\phi(x;\theta_i)-f_t(x) \right)^2=\frac{1}{m}\mathsf E\left(\alpha_t(\theta_1)\phi(x;\theta_1)-f_t(x)\right)^2\nonumber\\
&=\frac{1}{m}\left(\mathsf E\left(\alpha_t^2(\theta)\phi^2(x;\theta)\right)-f_t^2(x)\right). \end{align*}
Hence,
\begin{align}
&\mathsf E(\langle\widehat w_{t+1},\Phi_m(x_s)\rangle-f_{t+1}(x_s))^2 - \mathsf E(\langle\widehat w_t,\Phi_m(x_s)\rangle-f_t(x_s))^2 \nonumber\\
&= \frac{1}{m}\mathsf E((\alpha_{t+1}^2-\alpha_t^2)\phi^2(x_s;\cdot))-\frac{1}{m}(f^2_{t+1}(x_s)-f_t^2(x_s))  \nonumber\\
& \le \frac{a^2}{m}\mathsf E(|\alpha_{t+1}-\alpha_t|\cdot|\alpha_{t+1}+\alpha_t|)+\frac{1}{m}|f_{t+1}(x_s)-f_t(x_s)|\cdot |f_{t+1}(x_s)+f_t(x_s)| \nonumber\\
&\le \frac{a^2}{m}
\|\alpha_{t+1}-\alpha_t\|_{L^2(\mathsf P)} 
\|\alpha_{t+1}+\alpha_t\|_{L^2(\mathsf P)} 
+\frac{a^2}{m}\| f_{t+1}-f_t\|_\mathcal H\cdot \|f_{t+1}+ f_t\|_\mathcal H \nonumber\\
&=\frac{2a^2}{m}\| f_{t+1}-f_t\|_\mathcal H\cdot \|f_{t+1}+ f_t\|_\mathcal H\le\frac{4Ra^2}{m}\| f_{t+1}-f_t\|_\mathcal H, \label{3.1}
\end{align}
where we used the inequality \eqref{2.0} and Lemma \ref{lem:1}. Furthermore, 
\begin{align}
  |(f_{t+1}(x_s)-y_s)^2- (f_t(x_s)-y_s)^2|&=|f_{t+1}(x_s)-f_t(x_s)|\cdot |f_{t+1}(x_s)+f_t(x_s)+2 y_s|\nonumber \\
  &\le a(2aR+2Y)\|f_{t+1}-f_t\|_{\mathcal H}  \label{3.2}
\end{align}
Finally, using the decomposition,
\begin{align*}
\mathsf E(\langle\widehat w_t,\Phi(x)\rangle-y)^2 &=\mathsf E(\langle\widehat w_t,\Phi(x)\rangle-f_t(x))^2+(f_t(x)-y)^2,
\end{align*}
we combine \eqref{3.1} and \eqref{3.2}: 
\begin{align*}
    \mathsf E (\ell_s(\widehat w_{t+1})-\ell_s(\widehat w_t))&= \frac{1}{2}\mathsf E(\langle\widehat w_{t+1},\Phi_m(x_s)\rangle-y_s)^2-\frac{1}{2}\mathsf E(\langle\widehat w_t,\Phi_m(x_s)\rangle-y_s)^2\\
    &=\frac{1}{2}\mathsf E(\langle\widehat w_{t+1},\Phi_m(x_s)\rangle-f_{t+1}(x_s))^2-\frac{1}{2}\mathsf E(\langle\widehat w_t,\Phi_m(x_s)\rangle-f_t(x_s))^2\\
    &+ \frac{1}{2}(f_{t+1}(x_s)-y_s)^2- \frac{1}{2} (f_t(x_s)-y_s)^2\\
    &\le \left(\frac{2Ra^2}{m}+a(aR+Y)\right)\|f_{t+1}-f_t\|_{\mathcal H}.      \qedhere 
\end{align*}
\end{proof}

Using Lemma \ref{lem:3}, we bound the comparator variation term from Theorem \ref{th:1}:
\begin{align*}
\mathsf E [F_t^\gamma(\widehat w_{t+1}) - F_t^\gamma(\widehat w_t)]  
&=\gamma^t\frac{\lambda}{2}(\mathsf E\|\widehat w_{t+1}\|^2-\mathsf E\|\widehat w_t\|^2)+\sum_{s=1}^t\gamma^{t-s}(\mathsf E\ell_s(\widehat w_{t+1})-\mathsf E \ell_s(\widehat w_t))\\
&\le \gamma^t\frac{\lambda}{2} (\|f_{t+1}\|^2_\mathcal H-\|f_t\|^2_\mathcal H ) + \frac{\rho_m}{1-\gamma}\|f_{t+1}-f_t\|_{\mathcal H}\\
&\le\left(\lambda R+\frac{\rho_m}{1-\gamma}\right)\|f_{t+1}-f_t\|_{\mathcal H},
\end{align*}
since $\gamma^t \le 1$, $\|f_{t+1}\|^2_\mathcal H-\|f_t\|^2_\mathcal H \le 2R\|f_{t+1}-f_t\|_\mathcal H$, and $\sum_{s=1}^t \gamma^{t-s} < \frac{1}{1-\gamma}$.
Summing over $t$ gives
\begin{align*}
\gamma \mathsf E \sum_{t=1}^{T-1} [F_t^\gamma(\widehat w_{t+1}) - F_t^\gamma(\widehat w_t)]\le \left[\gamma\lambda R+\frac{\gamma}{1-\gamma}\rho_m \right]  P_T(\mathbf f),
\end{align*}
where $P_T(\mathbf f)=\sum_{t=1}^{T-1}\|f_{t+1}-f_t\|_\mathcal H$ is the path length of the comparator sequence $\mathbf f=(f_1,\dots,f_T)$.

Now we apply Theorem \ref{th:1} with the comparator sequence $\widehat{\mathbf w}=(\widehat w_1,\dots,\widehat w_T)$ and take the expectation. Using the inequalities,
\[ \sum_{t=1}^T \gamma^{T-t} \|\Phi(x_t)\|_2^2 \le a^2 T,\quad \ln \frac{1}{\gamma} \le \frac{1}{\gamma} - 1, \]
and recalling that by Lemma \ref{lem:2}, $\mathsf E\|\widehat w\|_2^2=\|f\|^2_\mathcal H\le R^2$, we obtain:
\begin{align}
\mathsf E R_T(\widehat {\mathbf w}) &\le \frac{\gamma\lambda}{2} \mathsf E\|\widehat w_1\|_2^2 + \frac{m}{2} \max_{1\le t\le T} \Delta_t^2 \ln \left( 1 + \frac{a^2 T}{\lambda m} \right) \nonumber\\
& + \gamma \sum_{t=1}^{T-1} \mathsf E[F_t^\gamma(\widehat w_{t+1}) - F_t^\gamma(\widehat w_t)] + \frac{m}{2} \Delta_{1:T}^2 \ln\left(\frac{1}{\gamma}\right) \nonumber \\
&\le \frac{\lambda}{2} R^2+\frac{m}{2} \max_{1\le t\le T} \Delta_t^2 \ln \left( 1 + \frac{a^2 T}{\lambda m} \right) \nonumber\\
&+ \left[\lambda R+\frac{\gamma}{1-\gamma} \rho_m \right]  P_T(\mathbf f) + \frac{m}{2} \frac{1-\gamma}{\gamma} \Delta_{1:T}^2 \nonumber\\
&= \eta \rho_m P_T(\mathbf f)+\frac{m}{2\eta}\Delta_{1:T}^2+\lambda R P_T(\mathbf f) \nonumber \\
&+\frac{\lambda}{2} R^2+\frac{m}{2} \max_{1\le t\le T} \Delta_t^2 \ln \left( 1 + \frac{a^2 T}{\lambda m} \right),\quad \eta=\frac{\gamma}{1-\gamma}. \label{3.3}
\end{align}

Let us optimize the dominant terms that depend on $\eta$. Assume $P_T(\mathbf f)>0$ and consider minimizing the function
\[ \psi(\eta):=\eta \rho_m P_T(\mathbf f) + \frac{m}{2\eta} \Delta_{1:T}^2 \quad \text{over } \eta\in (0,\infty).
\]
Setting the derivative $\psi'(\eta)$ to zero yields the optimal oracle choice for $\eta$:
\[\eta_*=\sqrt{\frac{m}{2}\frac{\Delta_{1:T}^2}{\rho_m P_T(\mathbf f)}}, \quad \text{which gives} \quad 
\psi(\eta_*)=\sqrt{2m \rho_m P_T(\mathbf f) \Delta_{1:T}^2}.\]
Substituting this back, we get an oracle bound on the expected regret against $\widehat{\mathbf w}$:
\begin{align*}
\mathsf E R_T(\widehat{\mathbf w}) &\le \sqrt{2m \rho_m P_T(\mathbf f) \Delta_{1:T}^2}+\lambda R P_T(\mathbf f) \\
&+\frac{\lambda}{2} R^2+\frac{m}{2} \max_{1\le t\le T} \Delta_t^2 \ln \left( 1 + \frac{a^2 T}{\lambda m} \right).
\end{align*}
This bound remains true also for $P_T(\mathbf f)=0$, since in this case $\eta^*=\infty$ (that is, $\gamma^*=1)$.

The oracle formula for $\eta_*$ cannot be used directly: it requires not only the knowledge of future values of $y_t$, but also a comparator sequence $\mathbf f$. However, following \cite{Jacobsen2024}, we can construct an algorithm that aggregates experts with different $\eta$ values, and achieves a nearly optimal bound. Let us define a grid of $\eta$ values:
$$ b>1,\quad \eta_{\min}=2m,\quad \eta_{\max}=mT,$$
$$\mathcal S_\eta=\{\eta_i=\eta_{\min} b^i\wedge\eta_{\max}:i\in\mathbb Z_+\},$$
$$\mathcal S_\gamma=\left\{\gamma_i=\frac{\eta_i}{1+\eta_i}:i\in\mathbb Z_+\right\}\cup\{0\}$$
exactly as in \cite{Jacobsen2024}.
\begin{lemma} \label{lem:4}
Let $\mathcal A_m(\lambda)$ be the VAW meta-algorithm that aggregates the predictions of DVAW forecasters $\{\mathcal A_{m,\gamma_k}(\overline\lambda) : \gamma_k \in \mathcal S_\gamma\}$. Then, the expected regret of $\mathcal A_m(\lambda)$ against a functional comparator sequence $\mathbf{f}$ is bounded by:
\begin{align}
\mathsf E R_T^{\mathcal A_m(\lambda)}(\mathbf f) &\le (1+b)\sqrt{\frac{m \rho_m P_T(\mathbf f) \Delta_{1:T}^2}{2}}+ \frac{1}{2} (Y+\tilde Y)^2+\overline\lambda R P_T(\mathbf f) + \frac{\overline\lambda R^2}{2} \nonumber \nonumber\\
& \quad + \frac{m}{2} (Y+\widetilde Y)^2  \ln\left(1+\frac{a^2 T}{\overline\lambda m}\right) + \frac{a^2 R^2 T}{2m}+\frac{\lambda}{2} + \frac{M Y^2 }{2} \ln \left(1+\frac{Z_{T,m}^2}{\lambda M}\right). \label{bound_of_lemma_4} 
\end{align}
where
\[ Z_{T,m}^2=\left[(M-1)((Y+\widetilde Y)^2+4Y^2)+\widetilde Y^2)\right]T+2(M-1)(Y+\widetilde Y)^2 m \ln\left(1+\frac{a^2 T}{\lambda} \right)    \]
and $M=O(\log_b(\eta_{\max}/\eta_{\min}))=O(\log_b T)$ is the cardinality of $\mathcal S_\eta$.
\end{lemma}

\begin{proof} The proof follows a standard meta-regret decomposition. For any oracle-chosen expert $k \in \{0, \dots, M-1\}$, the regret of the meta-algorithm $\mathcal{A}_m(\lambda)$ is:
\begin{align*}
    R_T^{\mathcal A_m(\lambda)}(\mathbf f) &= \underbrace{R_T^{\mathcal A_m(\lambda)}(e_k)}_{\text{meta regret}} + \underbrace{R_T^{\mathcal A_{m,\gamma_k}(\overline\lambda)}(\mathbf f)}_{\text{expert regret}},\\
    R^{\mathcal A_m(\lambda)}_T(e_k)&=\frac{1}{2}\sum_{t=1}^T (\langle z_t,\alpha_t\rangle-y_t)^2-\frac{1}{2}\sum_{t=1}^T (\langle z_t,e_k\rangle-y_t)^2,\\
    R_T^{\mathcal A_{m,\gamma_k}(\overline\lambda)}(\mathbf f)&=\frac{1}{2}\sum_{t=1}^T (z_{t,k}-y_t)^2-\frac{1}{2}\sum_{t=1}^T (f_t(x_t)-y_t)^2)
    \end{align*}
Here $z_t=(z_{t,0},\dots,z_{t,M-1})$ are the predictions of DVAW experts, and $\alpha_t$ is generated by $\mathcal A_m(\lambda)$. 

(1) \emph{Bounding the meta-regret}. The meta-algorithm is the standard VAW. Its regret against a fixed expert is bounded by \eqref{2.4}:
\begin{align*}
R^{\mathcal A_m}_T(e_k)\le \frac{\lambda}{2}\|e_k\|_2^2 + \frac{M Y^2 }{2} \ln \left(1+\frac{1}{\lambda M}\sum_{t=1}^T\|z_t\|_2^2\right).
\end{align*}
We need a bound for $\sum_{t=1}^T\|z_t\|_2^2=\sum_{t=1}^T z_{t,0}^2+\sum_{t=1}^{M-1}\sum_{t=1}^T z_{t,k}^2.$
The expert $k=0$ predicts $z_{t,0}=\widetilde y_t^2$, so
\[ \sum_{t=1}^T z_{t,0}^2\le T\widetilde Y^2.\]
For any expert $k\ge 1$ we have
\begin{align*}
   \sum_{t=1}^T z_{t,k}^2 &=\sum_{t=1}^T \langle w_{t,k},\Phi(x_t)\rangle^2\le 2\sum_{t=1}^T (\langle w_{t,k},\Phi(x_t)\rangle-y_t)^2+2\sum_{t=1}^T y_t^2=4 R_T^{\mathcal A_{m,\gamma_k}(\overline\lambda)}(\mathbf 0)+4\sum_{t=1}^T y_t^2.
\end{align*}
By Theorem~\ref{th:1},
\begin{align*}
    R_T^{\mathcal{A}_{m,\gamma_k}(\overline\lambda)}(\mathbf{0}) &\le \frac{m}{2} \max_{1\le t\le T} \Delta_t^2 \ln \left( 1 + \frac{\sum_{t=1}^T \gamma_k^{T-t} a^2}{\overline \lambda m} \right) + \frac{m}{2} \Delta_{1:T}^2 \ln\frac{1}{\gamma_k} \\
    &\le \frac{m(Y+\tilde{Y})^2}{2} \ln\left(1+\frac{a^2T}{\overline\lambda m}\right) + \frac{T (Y+\tilde{Y})^2}{4},
\end{align*}
where we used the inequality
\[ \ln\frac{1}{\gamma_k}\le\frac{1-\gamma_k}{\gamma_k}=\frac{1}{\eta_k}\le\frac{1}{\eta_{\min}}=\frac{1}{2m}.\]
Summing over $M-1$ such experts, and adding the term for $k=0$, we get:
\begin{align*}
\sum_{t=1}^T \sum_{k=0}^{M-1} z_{t,k}^2 &\le T\tilde{Y}^2 + (M-1)\left[ 2m(Y+\tilde{Y})^2 \ln\left(1+\frac{a^2T}{\overline\lambda m}\right) + T((Y+\tilde{Y})^2 + 4Y^2) \right] \\
&:= Z_{T,m}^2.
\end{align*}
Thus the meta-regret is bounded by
\begin{align} \label{3.4}
R^{\mathcal A_m}_T(e_k)\le \frac{\lambda}{2} + \frac{M Y^2 }{2} \ln \left(1+\frac{Z_{T,m}^2}{\lambda M}\right)
\end{align}
uniformly in $k$.

(2) \emph{Bounding the expert regret.} Consider a sequence $\widehat{\mathbf w}=(\widehat w_1,\dots,\widehat w_T)$, where $\widehat w_t$ is related to $f_t$ as in Lemma \ref{lem:2}:
$f(x)=\mathsf E[\alpha_t\phi(x;\cdot)]$,
 \[ \widehat w=\frac{1}{\sqrt m}\left(\alpha_t(\theta_i),\dots,\alpha_t(\theta_m) \right)^\top,\]
 where $\theta_i\sim \mathsf P$ are i.i.d. random variables. For $k\ge 1$ the expectation of expert's regret can be further decomposed as follows:
 \begin{align*}
\mathsf E R_T^{\mathcal A_{m,\gamma_k}(\overline\lambda)}(\mathbf f) & =\mathsf E \sum_{t=1}^T (\ell_t(w_{t,k}) - \ell_t(\widehat w_t))+\mathsf E\sum_{t=1}^T(\ell_t(\widehat w_t)-\frac{1}{2}(f_t(x_t)-y_t)^2)\\
&\le \mathsf E R_T^{\mathcal A_{m,\gamma_k}(\overline\lambda)}(\widehat {\mathbf w})+\frac{a^2}{2m}\sum_{t=1}^T\| f_t\|_\mathcal H^2 \le \mathsf E R_T^{\mathcal A_{m,\gamma_k}(\overline\lambda)}(\widehat {\mathbf w}) + a^2 R^2 \frac{T}{2m}.
\end{align*}
We used the inequality \eqref{2.2}. The first term is the regret of DVAW against the vector comparator $\widehat{\mathbf w}$, which is bounded in (\ref{3.3}):
\begin{align}
\mathsf E R_T^{\mathcal A_{m,\gamma_k}(\overline\lambda)}(\mathbf f) &\le \underbrace{\left( \eta_k \rho_m P_T(\mathbf f)+\frac{m}{2\eta_k}\Delta_{1:T}^2 \right)}_{\psi(\eta_k)} + \overline\lambda R P_T(\mathbf f) + \frac{\overline\lambda R^2}{2} \nonumber \\
& \quad + \frac{m}{2} \max_{1\le t\le T} \Delta_t^2 \ln\left(1+\frac{a^2 T}{\overline\lambda m}\right) + \frac{a^2 R^2 T}{2m}. \label{3.5}
\end{align}

(3) \emph{Choosing the best expert}. Since the meta-regret is bounded uniformly in $k$: see (\ref{3.4}), we only need to select an oracle expert $k$ to control the bound for the expected expert regret $\mathsf E R_T^{\mathcal A_m(\lambda,\gamma_k)}(\mathbf f)$. 

\emph{Case 1:} $\eta_*\le\eta_{\min}=2m$. We choose the special expert $k=0$, which corresponds to $\gamma_0=0$ and predicts $\widetilde y_t$. The expert regret is simply 
\[R_T^{\mathcal{A}_{m,0}(\overline\lambda)}(\mathbf{f}) = \frac{1}{2}\sum (\tilde{y}_t - y_t)^2 - \frac{1}{2}\sum(f_t(x_t)-y_t)^2 \le \frac{1}{2}\Delta_{1:T}^2.\]
The condition 
\[ \eta_*=\sqrt{\frac{m}{2}\frac{\Delta_{1:T}^2}{\rho_m P_T(\mathbf f)}} \le 2m\] 
implies that
\begin{align} \label{3.6}
 R_T^{\mathcal{A}_{m,0}(\overline\lambda)}(\mathbf{f})\le \sqrt{\frac{1}{2}\Delta_{1:T}^2} \sqrt{\frac{1}{2}\Delta_{1:T}^2}\le 2\sqrt{\frac{m\rho_m P_T(\mathbf f)\Delta_{1:T}^2}{2}. }   
\end{align}

\emph{Case 2:} $\eta_{\min} < \eta_* \le \eta_{\max}$.
Take $k$ such that $\eta_k\le\eta^*\le b\eta_k$. Then
\begin{align*}
\psi(\eta_k) & = \eta_k \rho_m P_T(\mathbf f)+\frac{m}{2\eta_k}\Delta_{1:T}^2  \le \eta_* \rho_m P_T(\mathbf f) + \frac{b m}{2\eta_*} \Delta_{1:T}^2 \\
&= \sqrt{\frac{m \rho_m P_T(\mathbf f) \Delta_{1:T}^2}{2}} + b\sqrt{\frac{m \rho_m P_T(\mathbf f) \Delta_{1:T}^2}{2}} = (1+b)\sqrt{\frac{m \rho_m P_T(\mathbf f) \Delta_{1:T}^2}{2}}.
\end{align*}

\emph{Case 3:} $\eta_* > \eta_{\max} = mT$.
Take largest $\eta_k = \eta_{\max}$. Then
\begin{align*}
\psi(\eta_k) &= \eta_{\max} \rho_m P_T(\mathbf f)+\frac{m}{2\eta_{\max}}\Delta_{1:T}^2\le \eta_* \rho_m P_T(\mathbf f) + \frac{1}{2 T} \Delta_{1:T}^2\\
&\le\sqrt{\frac{m\rho_m P_T \Delta_{1:T}^2}{2}} + \frac{1}{2} (Y+\tilde Y)^2.
\end{align*}

Taking the common upper bound for $\psi(\eta_k)$ and using (\ref{3.5}), we conclude that for $\eta_*>\eta_{\min}$ there exists $k\ge 1$ such that
\begin{align}
\mathsf E R_T^{\mathcal A_{m,\gamma_k}(\overline\lambda)}(\mathbf f) &\le (1+b)\sqrt{\frac{m \rho_m P_T(\mathbf f) \Delta_{1:T}^2}{2}}+ \frac{1}{2} (Y+\tilde Y)^2+\overline\lambda R P_T(\mathbf f) + \frac{\overline\lambda R^2}{2} \nonumber \\
& \quad + \frac{m}{2} \max_{1\le t\le T} \Delta_t^2 \ln\left(1+\frac{a^2 T}{\overline\lambda m}\right) + \frac{a^2 R^2 T}{2m}. \label{3.7}
\end{align}
Moreover, from (\ref{3.6}) it now follows that for any $\eta_*$ there exists $k\ge 0$ such that the same inequality \eqref{3.7} holds true. It remains to combine \eqref{3.4} and \eqref{3.7} to get the final result. 
\end{proof}

In comparison to \cite{Jacobsen2024} a key simplification is our use of the standard path length $P_T(\mathbf f)$. This avoids a complicated coupling between the discount factor $\gamma$ and the comparator variation measure present in the original DVAW bound. On the other hand, \cite{Jacobsen2024} works with finite-dimensional comparator sequences.

\begin{theorem} \label{th:2}
Let the H-VAW-D algorithm, denoted by $\mathcal A$, be the top-level VAW meta-algorithm with $\lambda=1$ that aggregates predictions from the set of experts $\{\mathcal{A}_m(\lambda): m \in \mathcal{S}_m\}$,
\[ \mathcal{S}_m = \{0\} \cup \{2^j : j=0, 1, \dots, \lceil\frac{1}{2}\log_2 T\rceil\}. \]
The expert $ \mathcal{A}_0$ is a simple hint-predictor ($\zeta_{t,0}=\tilde{y}_t$), and experts $\mathcal{A}_m(\lambda)$ for $m>0$ are the two-level algorithms from Lemma \ref{lem:4}, each build on $\mathcal \{A_{m,\gamma_k}(1/m)$: $\gamma_k\in S_\gamma$\}.  Under the assumptions $\|f_t\|_\mathcal{H} \le R$, $|y_t|\le Y$, $|\tilde{y}_t| \le \tilde{Y}$ the expected dynamic regret of $\mathcal{A}$ is bounded by:
\begin{align}
\mathsf{E} R_T^{\mathcal A}(\mathbf{f}) = O(\left((1+b)^2(1+a^2)\rho_\infty R^2 P_T(\mathbf f)\Delta_{1:T}^2 T)\right)^{1/3}+(Y+\widetilde Y)^2 \sqrt T\ln T+a^2 R^2 \sqrt T), \label{main_bound}
\end{align}
where $\rho_\infty=a(aR+Y).$ 
\end{theorem}
\begin{proof}
 Let $\zeta_{t,m}$ be the prediction of the Level 2 expert $\mathcal{A}_m(\lambda)$ at time $t$, and let $\zeta_t=\langle\beta_t,\zeta_t\rangle=\sum_{j=0}^m \beta_{t,j}\zeta_{t,j}$ be the final prediction of the top-level algorithm $\mathcal{A}$. For any choice of expert $m_* \in \mathcal{S}_m$, the total regret of $\mathcal{A}$ can be decomposed as:
\begin{align*}
R_T^{\mathcal{A}}(\mathbf{f}) &= \underbrace{R_T^{\mathcal{A}}(e_{m*})}_{\text{top-level meta-regret}} + \underbrace{R_T^{\mathcal{A}_{m_*}}(\mathbf{f})}_{\text{oracle expert regret}},\\
R_T^{\mathcal A}(e_m^*)&=\frac{1}{2}\sum_{t=1}^T(\langle\beta_t,\zeta_t\rangle-y_t)^2-\frac{1}{2}\sum_{t=1}^T(\zeta_{t,m^*}-y_t)^2,\\
R_T^{\mathcal{A}_{m_*}}(\mathbf{f})&=\frac{1}{2}\sum_{t=1}^T(\zeta_{t,m^*}-y_t)^2-\frac{1}{2}\sum_{t=1}^T(f_t(x_t)-y_t)^2.
\end{align*} 
For notational simplicity we do not show the dependence of the second-level expert algorithms $\mathcal A_m$ on $\lambda$ in the course of the proof.

\emph{(1) Bounding the top-level meta-regret.}
The top-level algorithm $\mathcal{A}$ is a VAW instance with parameter $\lambda=1$ and $N_m = |\mathcal{S}_m| = O(\ln T)$ experts. Its expected regret against any fixed expert $e_{m_*}$ is bounded by \eqref{2.4}: 
\begin{align} \label{3.8}
\mathsf{E} R_T^{\mathcal{A}}(e_{m_*}) \le \frac{1}{2} + \frac{N_m Y^2}{2} \ln\left(1 + \frac{1}{N_m} \sum_{m \in \mathcal{S}_m} \sum_{t=1}^T \mathsf{E}[\zeta_{t,m}^2]\right)
\end{align}
by Jensen's inequality. The grid $\mathcal{S}_m$ has $N_m =  O(\ln T)$ experts. We bound the expected sum of squared predictions used in the VAW meta-regret bound:
\[
\sum_{m \in \mathcal{S}_m} \sum_{t=1}^T \mathsf{E}[\zeta_{t,m}^2] \le \sum_{t=1}^T \mathsf{E}[\zeta_{t,0}^2] + \sum_{m \in \mathcal{S}_m, m>0} \sum_{t=1}^T \mathsf{E}[\zeta_{t,m}^2].
\]
The first term is  bounded by $T\widetilde{Y}^2$. For the second sum, we use the bound 
\[\sum_{t=1}^T\mathsf{E}[\zeta_{t,m}^2] \le 2\sum_{t=1}^T\mathsf{E}[(\zeta_{t,m}-y_t)^2]+2\sum_{t=1}^Ty_t^2= 4\mathsf{E}R_T^{\mathcal{A}_m}(\mathbf{0}) + 4TY^2,\]
where from Lemma \ref{lem:4} with $P_T(\mathbf f)=0$, $R=0$, $\overline\lambda=1/m$, we get
\[\mathsf{E} R_T^{\mathcal{A}_m}(\mathbf{0}) = O(m(Y+\tilde{Y})^2 \ln T + Y^2\ln T).\]
Note that $\sum_{m \in \mathcal{S}'_m, m>0} m = \sum_{j=0}^{\lceil\frac{1}{2}\log_2 T\rceil} 2^j = O(\sqrt{T})$. Hence,
\begin{align*}
\sum_{m \in \mathcal{S}_m} \sum_{t=1}^T \mathsf{E}[\zeta_{t,m}^2] &=T\widetilde Y^2+\sum_{m\in \mathcal  S_m, m>0}(4\mathsf{E}R_T^{\mathcal{A}_m}(\mathbf{0}) + 4TY^2)\\
&=T\widetilde Y^2+O((Y+\tilde{Y})^2 \sqrt{T}\ln T+Y^2(\ln T)^2)+Y^2T\ln T)=O(T\ln T).
\end{align*}
Now \eqref{3.8} implies
\begin{align} \label{3.9}
\mathsf{E} R_T^{\mathcal{A}}(e_{m_*}) = O(Y^2(\ln T)^2).
\end{align}

\emph{(2) Bounding the expert regret in the static case.}
Let $P_T(\mathbf f)=0$. To balance the dominant terms 
\[ \frac{m}{2} (Y+\widetilde Y)^2  \ln\left(1+a^2 T\right) + \frac{a^2 R^2 T}{2m} \]
in \eqref{bound_of_lemma_4} take $m=m_s$ equal to the upper bound of the grid: $\sqrt T\le m_s=2^{N_m}\le 2\sqrt T.$ Then
\begin{align} \label{3.10}
\mathsf E R_T^{\mathcal A_{m_s}}(\mathbf f)=O((Y+\widetilde Y)^2 \sqrt T\ln T+a^2 R^2 \sqrt T).    
\end{align}

\emph{(3) Bounding the expert regret in the dynamic case.} Let $P_T(\mathbf f)>0$. First assume that $\widetilde y_t$ is a sequence of ``prefect'' hints:
\[ \frac{1}{2}\Delta_{1:T}^2=\frac{1}{2}\sum_{t=1}^T (y_t-\tilde y_t)^2\le Y^2\sqrt T\ln T. \]
Put $m=0$ then
\begin{align} \label{3.11}
 \mathsf E R_T^{\mathcal A_0}(\mathbf f)=\frac{1}{2}\sum_{t=1}^T(\tilde y_t-y_t)^2-\frac{1}{2}\sum_{t=1}^T(f_t(x_t)-y_t)^2=O(Y^2\sqrt T\ln T).
\end{align}
This bound is of the same order or better than the static rate derived above.

Finally, consider the main case of ``imperfect'' hints:
\begin{equation} \label{imperfect}
\frac{1}{2}\Delta_{1:T}^2\ge Y^2\sqrt T\ln T.     \end{equation}
Let us rewrite the bound \eqref{bound_of_lemma_4} explicitly:
\begin{align}
\mathsf E R_T^{\mathcal A_m}(\mathbf f) &\le (1+b)\sqrt{\frac{m \rho_m P_T(\mathbf f) \Delta_{1:T}^2}{2}}+ \frac{1}{m}\left(R P_T(\mathbf f)+\frac{a^2 R^2 T}{2} +\frac{R^2}{2}\right) + \frac{1}{2} (Y+\tilde Y)^2 \nonumber \nonumber\\
& \quad + \frac{m}{2} (Y+\widetilde Y)^2  \ln\left(1+a^2 T\right) + \frac{\lambda}{2} + \frac{M Y^2 }{2} \ln \left(1+\frac{Z_{T,m}^2}{\lambda M}\right). \label{3.12}
\end{align}
Consider the dominant term 
$$ g(m)=A\sqrt{m}+\frac{B}{m},\quad A=(1+b)\sqrt{\frac{\rho_\infty P_T(\mathbf f) \Delta_{1:T}^2}{2}},\quad B=R P_T(\mathbf f)+\frac{a^2 R^2 T}{2}, $$
$\rho_\infty=\lim_{m\to\infty} \rho_m=a(aR+Y). $ 
The global minimum $m_d$ of $g$ over $(0,\infty)$ exists, since $g$ tends to $\infty$ as $m\to +0$ and $m\to +\infty$. It coincides with the unique stationary point 
\begin{align*}
   m_d=\left(\frac{2B}{A}\right)^{2/3}=O\left(\left(\frac{T}{\sqrt{P_T(\mathbf f)\Delta_{1:T}^2}}\right)^{2/3}\right)=O\left(\left(\frac{T}{\sqrt{P_T(\mathbf f)\sqrt T\ln T}}\right)^{2/3}\right). 
\end{align*}
where we used the fact that $P_T(\mathbf f)=O(T)$ and the inequality \eqref{imperfect}. Since $P_T(\mathbf{f}) > 0$, we have $P_T(\mathbf{f}) = \Omega(1)$. Thus
\[ m_d=O\left(\frac{T^{3/4}}{(\ln T)^{1/2}} \right)^{2/3}=O\left(\frac{T^{1/2}}{(\ln T)^{1/3}}\right)=o(T^{1/2}).\]
This proves that $m_d\le N_m$ for sufficiently large $T$.
Take an expert $m'_d\in \mathcal{S}_m$ that provides a multiplicative approximation to $m_d$:
\[
m_d/2 < m'_d \le m_d.
\]
Substituting $m_d'$ into \eqref{3.12}, we get
\begin{align*}
\mathsf E R_T^{\mathcal A_{m_d'}}(\mathbf f) &=O\left( (1+b)\sqrt{m_d' \rho_\infty P_T(\mathbf f) \Delta_{1:T}^2}+ \frac{1}{m_d'}\left(R P_T(\mathbf f)+a^2 R^2 T \right) \right. \nonumber\\
& \left. \quad + m_d' (Y+\widetilde Y)^2  \ln\left(1+a^2 T\right)  + M Y^2 \ln \left(1+\frac{Z_{T,m_d'}^2}{\lambda M}\right) \right)\\
&=O\left(g(m_d)+m_d(Y+\widetilde Y)^2  \ln T+M Y^2 \ln Z_{T,m_d}\right). 
\end{align*}
Here
\begin{align*}
 g(m_d)&=O((A^2 B)^{1/3})=O\left(\left((1+b)^2\rho_\infty P_T(\mathbf f)\Delta_{1:T}^2(RP_T(\mathbf f)+aR^2 T)\right)^{1/3}\right)\\
 &=O\left(\left((1+b)^2(1+a^2)\rho_\infty R^2 P_T(\mathbf f)\Delta_{1:T}^2 T)\right)^{1/3}\right), 
\end{align*}
and other terms are bounded better than the expected regret in the static case:
\[m_d(Y+\widetilde Y)^2  \ln T+M Y^2 \ln Z_{T,m_d}=(Y+\widetilde Y)^2 o(\sqrt T)+M Y^2 O(\ln T).\]
Combining this result with \eqref{3.10}, \eqref{3.11}
we can write the common upper bound for the oracle expert:
\begin{align*}
R_T^{\mathcal{A}_{m_*}}(\mathbf{f})=O(\left((1+b)^2(1+a^2)\rho_\infty R^2 P_T(\mathbf f)\Delta_{1:T}^2 T)\right)^{1/3}+(Y+\widetilde Y)^2 \sqrt T\ln T+a^2 R^2 \sqrt T).  
\end{align*}
This bound is the same as \eqref{main_bound}, since \eqref{3.9} is of a lower order.
\end{proof}

\begin{remark}
The bound \eqref{main_bound} consists of a dynamic component dependent on the path length $P_T(\mathbf{f})$ and a static component. The last two terms in \eqref{main_bound} are related to the static regret. Although formally the last term can be dropped due to the $\ln T$ factor in the second-to-last term, we retain it. This is because in practical scenarios, the norm bound $R$ for the comparator sequence can be large relative to the label and hint bounds $Y$ and $\widetilde{Y}$. Thus, the third term can be larger than the second one for realistic time horizons, and keeping it provides a more complete characterization of the algorithm's performance. However, this bound can also be written in a shorter form (\ref{1.1}), since $\Delta_{1:T}^2=O(T)$.  
\end{remark}

\begin{remark}
The H-VAW-D algorithm has a hierarchical structure. To estimate its computational complexity note that a single DVAW expert $\mathcal{A}_{m,\gamma_k}$ operating on $m$ random features has a computational cost $O(m^2)$ due to the updating rule \eqref{inverse_Sigma_update}. The meta-algorithm $\mathcal A_m$ runs $M=|\mathcal S_\gamma|=O(\ln T)$ base DVAW experts, and updates its own VAW state using the $M$-dimensional vector of their predictions. The total per-step cost of $\mathcal A_m$ is the sum of these components:
\[\operatorname{Cost}(\mathcal A_m)=O(m^2 M+ M^2)=O(m^2\ln T+\ln^2 T).\]

Furthermore, $\mathcal A$ VAW meta-algorithm runs all $\mathcal A_m$ experts for $m\in\mathcal S_m$ and updates its own $N_m$ dimensional state, where $N_m=|\mathcal S_m|=O(\ln T)$. Thus, the per-step complexity of $\mathcal A$ is estimated by
\
\begin{align*}
\operatorname{Cost}(\mathcal A)&=\sum_{m\in\mathcal S_m}\operatorname{Cost}(\mathcal A_m)+O(N_m^2)
= \sum_{j=1}^{\lceil\frac{1}{2}\log_2 T\rceil} O((2^j)^2 \ln T+\ln^2 T) + O(\ln^2 T)\\
&=O( 4^{\frac{1}{2}\log_2 T}\ln T)+O(\ln^3 T)=O(T\ln T).
\end{align*}
\end{remark}

\section{Conclusion}\label{sec:conclusion}
In this work we addressed the challenging problem of non-parametric non-stationary online regression in an adversarial setting. By synthesizing the discounted Vovk-Azoury-Warmuth (DVAW) framework of \cite{Jacobsen2024} with a random feature approximation, we developed the H-VAW-D algorithm,
a hierarchical forecaster capable of competing against a time-varying sequence of functions in a reproducing kernel Hilbert space. Our theoretical analysis shows that this fully adaptive method achieves an expected dynamic regret of $O(T^{2/3}P_T^{1/3} + \sqrt{T}\log T)$, where $P_T$ is the functional path length of a comparator sequence.

There are several possibilities for future research. First, while our method avoids the $O(T^3)$ cost of exact kernelization, its $O(T^2 \ln T)$ overall complexity is still very large. This cost is a direct consequence of maintaining a growing grid of experts for each discount factor and feature dimension. It may be possible to mitigate this complexity by applying the Follow-the-Leading-History (FLH) algorithm developed for adaptive regret in \cite{Hazan2009efficient}. Second direction, related to the previous one, concerns the computer experiments with the proposed algorithm on non-stationary data. Finally, a key theoretical question is the optimality of our regret bound. Establishing a lower bound for dynamic regret in the RKHS setting is essential to determine whether the $O(T^{2/3}P_T^{1/3})$ is optimal for this problem class.      

\printbibliography

\end{document}